\documentclass[12pt]{article}
\usepackage{fullpage}

\usepackage{amsmath,amsfonts,amssymb}
\usepackage{algorithm,algorithmic}
\usepackage{subfigure}
\usepackage{tikz,pgfplots}
\usepackage{framed}

\usepackage{amsthm}

\usepackage[round]{natbib}

\newtheorem{theorem}{Theorem}[section]
\newtheorem{definition}{Definition}[section]
\newtheorem{lemma}{Lemma}[section]
\newtheorem{corollary}{Corollary}[section]

\newcommand{\vv}{\ensuremath{\mathbf v}}
\newcommand{\uu}{\ensuremath{\mathbf u}}

\DeclareMathOperator*{\argmin}{arg\,min}

\def\reals{{\mathcal R}}
\def\fhat{{\ensuremath{\hat{f}}}}

\newcommand{\K}{\mathcal{K}}

\newcommand{\ignore}[1]{}

\def\reals{{\mathbb R}}

\def\bold0{\mathbf{0}}

\newcommand\E{\mbox{\bf E}}

\def\x{\mathbf{x}}
\def\y{\mathbf{y}}

\newcommand{\B}{\mathbb{B}}
\newcommand{\Sp}{\mathbb{S}}

\newcommand{\eps}{\varepsilon}

\usepackage{hyperref}
\hypersetup{colorlinks,
            linkcolor=blue,
            citecolor=blue,
            urlcolor=magenta,
            linktocpage,
            plainpages=false}

\title{On Graduated Optimization for  Stochastic\\ Non-Convex Problems }

\author{%
Elad Hazan\footnote{Princeton University; \texttt{ehazan@cs.princeton.edu}.}
\and
Kfir Y. Levy\footnote{Technion; \texttt{kfiryl@tx.technion.ac.il}.}
\and 
Shai Shalev-Shwartz\footnote{The Hebrew University; \texttt{shais@cs.huji.ac.il}.}
}

\date{May 2014}                                           

\begin{document}
\maketitle

\begin{abstract} 
The graduated optimization approach, also known as the continuation method, is a popular heuristic to solving non-convex problems that has received renewed interest over the last decade. 
Despite its popularity, very little is known in terms of theoretical convergence analysis.
 
In this paper we describe a new first-order algorithm based on graduated optimization and analyze its performance. We characterize a parameterized family of non-convex functions for which this algorithm provably converges to a {\bf global} optimum. In particular, we prove that the algorithm  converges to an $\eps$-approximate solution within $O(1 /  \eps^2)$ gradient-based steps. 
We extend our algorithm and  analysis  to the setting of stochastic non-convex optimization with  noisy gradient feedback, attaining the same convergence rate.
Additionally, we discuss the  setting of ``zero-order optimization", and devise a a variant of our algorithm  which converges at rate of  $O(d^2/  \eps^4)$.
 
 \end{abstract} 

\section{Introduction}
\label{section:intro}

Non-convex optimization programs are ubiquitous in machine learning and computer vision.  Of particular interest are non-convex optimization problem that arise in the training of deep neural networks \cite{bengio2009learning}.  Often, such problems admit a multimodal structure, and therefore, the use of convex optimization machinery  may lead to a local optima. 

Graduated optimization (a.k.a. continuation), \cite{blake1987visual}, is a methodology that attempts to overcome such numerous local optima.   At first, a  coarse-grained  version of the problem is generated by a local smoothing operation. This coarse-grained version  is easier to solve.  Then, the method  advances in stages by  gradually refining the problem versions,  using the solution of the previous stage as an initial point for the optimization in the next stage. 

Despite its popularity, there are still many gaps concerning both theoretical and practical aspects of  graduated optimization, and in particular we are not aware of a rigorous  running time analysis to find a global optimum, or even conditions in which a global optimum is reached.  Nor are we familiar with graudated optimization in the stochastic setting, in which only a noisy gradient or value oracle to the objective is given. 
Moreover, any practical application of graduated optimization requires an efficient construction of coarse-grained versions of the original function. 
For some special cases this construction can be made analytically \cite{chapelle2006continuation, chaudhuri2011smoothing} . However, in the general case,
 it is commonly suggested in the literature  to convolve the original function with a gaussian kernel \cite{wu1996effective}. Yet, this operation is prohibitively inefficient in high dimensions.

In this paper we take an  algorithmic / analytic  approach to  graduated optimization and show the following. 
\begin{itemize}
\item
We characterise a family of non-convex multimodal functions that allows convergence to a {\it global} optimum. This parametrized family we call $\sigma$-nice (see Definition~\ref{def:SigmaNiceDef} ).
\item
We provide a stochastic algorithm inspired by graduated optimization,  that performs only gradient updates and is ensured to find an $\eps$-optimal solution of $\sigma$-nice functions within $O(1/ \sigma^2 \eps^2)$ iterations.  The algorithm doesn't require expensive convolutions and access the smoothed version of any function using random sampling.  
The algorithm only requires  access to the objective function through a \emph{noisy} gradient oracle.

\item
We extend our method to the ``zero-order optimization" model (a.k.a. ``bandit feedback" model), in which the objective is only accessible through a noisy value oracle. We devise a variant of our algorithm that is guaranteed  to find an $\eps$-optimal solution within $O(d^2/\sigma^2 \eps^4)$ iterations. 

\end{itemize}

Interestingly, the next question is raised in \cite{bengio2009learning} which reviews recent developments in the field of deep learning:
\textbf{``Can optimization strategies based on continuation methods deliver significantly improved training of deep architectures?"}

As an initial empirical study, we examine the task of training a NN (Neural Network) over the MNIST data set. 
Our experiments support the theoretical guarantees, demonstrating that graduated optimization according to the methodology proposed  accelerates convergence in training the NN. 
Moreover, we show examples in which  $\sigma$-nice functions capture  non-convex structure/phenomena that exists in natural data.

\subsection{Related Work} 
Among the machine vision community, the idea of graduated optimization was known since the 80's. The term
``Graduated Non-Convexity" (GNC) was  coined by \cite{blake1987visual},  who were  the first to establish this idea explicitly. 
Similar attitudes in the  machine vision literature  appeared later in \cite{yuille1989energy, yuille1990stereo}, and \cite{terzopoulos1988computation}.

Concepts of the same nature appeared in the optimization literature~\cite{wu1996effective}, and in the field of  numerical analysis \cite{allgower1990numerical}.

Over the last two decades, this concept was successfully applied to numerous problems in computer vision; among are: image deblurring \cite{boccuto2002gnc}
, image restoration \cite{nikolova2010fast}, and optical flow \cite{brox2011large}.
The method was also adopted by the machine learning community, demonstrating effective performance in tasks 
such as semi-supervised learning \cite{chapelle2006continuation}, graph matching \cite{zaslavskiy2009path}, and ranking \cite{chapelle2010gradient}.
In \cite{bengio2009learning}, it is suggested to consider some developments in deep belief architectures \cite{hinton2006fast,erhan2009difficulty} as a kind of continuation. These approaches, in the  spirit of the continuation method,  offer no guarantees on the quality of the obtained solution, 
and are tailored to specific applications.

A comprehensive survey of the graduated optimization literature can be found in \cite{mobahi2015link}.

A recent work \cite{mobahi2015theoretical} 
advances our theoretical understanding, by analyzing a continuation algorithm in the general setting. Yet, they offer no way to perform the smoothing efficiently,
nor a way to optimize the smoothed versions; but rather assume that these are possible. Moreover, their guarantee is limited to a fixed precision  that depends on the  objective function and does not approach zero. In contrast, our approach can generate arbitrarily precise solutions.

\section{Setting}
\label{sec:setting}
We discuss an optimization of a  \emph{non-convex} loss function $f:\K\mapsto \reals$, where $\K \subseteq \reals^d$ is a convex set. We assume that optimization lasts for $T$ rounds; On each round $t=1,\ldots,T$, we may query a point $\x_t\in\K$, and receive a \emph{feedback}.
After the last round,  we choose $\bar{\x}_T\in\K$, and our performance measure is  the excess loss, defined as:
$$f(\bar{\x}_T)- \min_{\x\in\K}f(\x) $$
In Section~\ref{sec:SigmaNiceDef} we characterize a family of non-convex multimodal functions we call $\sigma$-nice.
Given such a $\sigma$-nice loss $f$, we are interested in algorithms that with a high probability ensure a $\eps$-excess loss within $\text{poly}(1/\eps)$ rounds. 

We consider two kinds of feedback:
\begin{enumerate}
\item \textbf{Noisy Gradient feedback:} Upon querying $\x_t$ we receive $\nabla f(\x_t)+\xi_t$, where $\{\xi_\tau\}_{\tau=1}^T$ are independent zero mean and bounded r.v.'s. 
\item \textbf{Noisy Value feedback (Bandit feedback):} Upon querying $\x_t$ we receive $f(\x_t)+\xi_t$, where $\{\xi_\tau\}_{\tau=1}^T$ are independent zero mean and bounded r.v.'s. 
\end{enumerate}

\section{Preliminaries and Notation}
\paragraph{Notation:} During this paper we use $\B,\Sp$ to denote the
unit Euclidean ball/sphere in $\reals^d$, and also $\B_r(\x),\Sp_r(\x)$ as the Euclidean $r$-ball/sphere in $\reals^d$ centered at $\x$. For a set $A\subset \reals^d$ , 
$\uu \sim A$ denotes a random variable distributed uniformly over $A$.
\subsection{Strong-Convexity}
Recall the definition of strongly-convex functions:
\begin{definition}\textbf{(Strong Convexity)}
We say that a function $F: \reals^n \to \reals$ is $\sigma$-strongly convex over the set $\K$ if for all $\x,\y \in \K$ it holds that,
\begin{align*}
F(\y) \geq F(\x) + \nabla F(\x)^\top(\y-\x) + \frac{\sigma}{2}\|\x - \y\|^2
\end{align*}
\end{definition}
Let $F$ be a $\sigma$-strongly convex over convex set $\K$, and let $\x^*$ be a point in $\K$ where $F$ is minimized, then the following inequality is satisfied:
\begin{align} \label{eq:strongCvxity}
\frac{\sigma}{2}\|\x-\x^*\|^2 \leq F(\x)-F(\x^*)
\end{align}
This is immediate by the definition of strong convexity combined with $\nabla F(\x^*)^\top(\x-\x^*)\geq 0,\; \forall \x\in \K $.

\section{Smoothing and $\sigma$-Nice functions}
Constructing finer and finer approximations to the original objective function is at the heart of  the continuation approach. 
In Section~\ref{sec:smoothing} we define the smoothed versions that we will employ. Next, in Section~\ref{sec:ImplicitSmoothing} we describe an efficient way to implicitly access the smoothed versions, which will enable us to perform optimization. Finally, in Section~\ref{sec:SigmaNiceDef} we define a class of non-convex multimodal functions we denote as \emph{$\sigma$-nice}. As we will see in Section~\ref{section:experiments}, these functions are rich enough to capture non-convex structure that exists in natural data. Additionally, these functions lend themselves to an efficient optimization, and we can ensure a convergence to $\eps$-solution within
 $\text{poly}(1/\eps)$ iterations, as described  in Sections~\ref{sec:SigmaNiceOptGrad},\ref{sec:SigmaNiceOptVal}.

\subsection{Smoothing}
\label{sec:smoothing}
Smoothing by local averaging is  formally defined next. 
\begin{definition}
Given an $L$-Lipschitz  function $f: \reals^d \mapsto \reals$ define it's $\delta$-smooth version to be
$$ \fhat_\delta(\x) = \E_{\uu \sim \B } [f(\x + \delta \uu) ]  .$$
\end{definition}
The next lemma bounds the bias between $\fhat_\delta$ and $f$.
\begin{lemma}
	\label{lem:SmoothingLemma}
	Let $\hat{f}_\delta$ be the $\delta$-smoothed version of $f$, then,
		 $$\forall \x \in \reals^d : |\hat{f} _\delta (\x) - f(\x) | \le \delta L$$
\end{lemma}

\begin{proof} [Proof of Lemma \ref{lem:SmoothingLemma}]
	\begin{align*}
		|\hat{f}_\delta (\x)-f(\x)|&=|\E_{u \sim \B } \left[  f(\x+ \delta \uu)\right] - f(\x) |  \\
		&\leq \E_{\uu \sim \B } \left[  |f(\x+ \delta \uu) - f(\x) |\right]   \\
		&\leq \E_{\uu \sim \B }\left[ L\| \delta \uu \| \right]   \\
		&\leq L \delta 
	\end{align*}
in the first inequality we used Jensen's inequality, and in the last inequality we used $\|\uu\|\leq 1$, since $\uu\in \B$.
\end{proof}

\subsubsection{Implicit Smoothing using Sampling}
\label{sec:ImplicitSmoothing}

A direct way to optimize a smoothed version is by direct calculation of its gradients, nevertheless this calculation might be very costly in high dimensions.
A much more efficient approach is to produce an unbiased estimate for the gradients of the smoothed version by sampling the function gradients/values. These estimates could then be used by a stochastic optimization algorithms such as SGD (Stochastic Gradient Descent).
This sampling approach is outlined in Figures~\ref{fig:SGO_G},\ref{fig:SGO_V}.

\begin{figure}[t]
\begin{framed}
 \textbf{Oracle 1}:  $\text{SGO}_G$\\
 \textbf{Input}:  $\x\in \reals^d$, smoothing parameter $\delta$\\
\textbf{Return}: $\nabla f(\x+\delta \uu)$, where $\uu\sim\B$
\end{framed}
\caption{ Smoothed gradient oracle given gradient feedback.}
\label{fig:SGO_G}
\end{figure}

\begin{figure}[t]
\begin{framed}
 \textbf{Oracle 2}:  $\text{SGO}_V$\\
 \textbf{Input}:  $\x\in \reals^d$, smoothing parameter $\delta$ \\
  \textbf{Return}: $\frac{d}{\delta}f(\x+\delta \vv) \vv $,  where $\vv\sim\Sp$
  \end{framed}
\caption{ Smoothed gradient oracle given value feedback.}
\label{fig:SGO_V}
\end{figure}

The following two Lemmas state that the resulting estimates are unbiased and bounded:
\begin{lemma} \label{lem:Unbiasedness1}
Let $\x\in\reals^d$, $\delta\geq 0$, and suppose that $f$ is $L$-Lipschitz, then the output of $\text{SGO}_G$ (Figure~\ref{fig:SGO_G}) is bounded by $L$ and is an unbiased estimate for $\nabla\fhat_\delta(\x)$. 
\end{lemma}
\begin{proof}
$\text{SGO}_G$  outputs $\nabla f(\x+\delta \uu) $ for some $\uu\in\B$, so  the first part is immediate by the Lipschitzness of $f$.
Now, by definition, $\fhat_{\delta}(\x) = \E_{\uu\sim \B}[f(\x+{\delta} \uu)]$, deriving both sides we get the second part of the Lemma.
\end{proof}

\begin{lemma} \label{lem:Unbiasedness2}
Let $\x\in\K\subseteq \reals^d$, $\delta\geq 0$, and suppose that $\max_\x |f(\x)|\leq C$, then the output of $\text{SGO}_V$ (Figure~\ref{fig:SGO_V})  is bounded by $\frac{dC}{\delta}$ and is an unbiased estimate for $\nabla\fhat_\delta(\x)$. 
\end{lemma}
\begin{proof}
$\text{SGO}_V$  outputs $\frac{d}{\delta}f(\x+\delta \vv) \vv $ for some $\vv\in\Sp$, since $f$ is $C$-Bounded over $\K$ the first part of the lemma is immediate.
In order to prove the second part, we can use Stokes theorem to show that if $\vv\sim \Sp$, then:
\begin{equation}
\label{eq:stokes_application}
\forall \x \in \reals^d \ . \ \E _{\vv  \sim \Sp} [f(\x + \delta \vv) \vv] = \frac{\delta}{d} \nabla\hat{f}_\delta (\x)
\end{equation}
A proof of Equation~\eqref{eq:stokes_application} is found in~\cite{Flaxman}.

\end{proof}
Note  that the oracles depicted in Figures~\ref{fig:SGO_G}, \ref{fig:SGO_V} may require sampling  function values outside $\K$, (specifically in $\K+\delta\B$).
 We assume that this is possible, and that the bounds over the function gradients/values inside $\K$, also apply in $\K+\delta\B$.
 
 \paragraph{Extensions to the  \emph{noisy} feedback settings:} 
 Note that for ease of notation, the oracles that appear in  Figures~\ref{fig:SGO_G}, \ref{fig:SGO_V}, assume we can access \emph{exact} gradients/values of $f$.
Given that we may only access \emph{noisy} and bounded gradient/value estimates of $f$ (Sec.~\ref{sec:setting}), we could use these instead of the exact ones that appear in Figures~\ref{fig:SGO_G},\ref{fig:SGO_V}, and still produce unbiased and bounded gradient estimates 
for the smoothed versions of $f$ as shown in Lemmas~\ref{lem:Unbiasedness1},\ref{lem:Unbiasedness2}.
 
 Particularly, in the case we may only access noisy gradients of $f$, then $\text{SGO}_G$ (Figure~\ref{fig:SGO_G}) will return $\nabla f(\x+\delta \uu)+\xi$ instead of 
$\nabla f(\x+\delta \uu)$, where $\xi$ is a noise term. Since we assume zero bias and bounded noise this implies that $\nabla f(\x+\delta \uu)+\xi$ is an unbiased estimate of $\nabla \fhat_\delta(\x)$, bounded by $L+K$ where $K$ is the bound on the noise and $L$ is the Lipschitz constant of $f$. 
We can show the same for $\text{SGO}_V$ (Figure~\ref{fig:SGO_V}), given a noisy value feedback.

\subsection{$\sigma$-Nice Functions}\label{section:sigmaNice}

\label{sec:SigmaNiceDef}
Following is our main definition 
\begin{definition}
\label{def:SigmaNiceDef}
 A function $f: \K \mapsto \reals$ is said to be
  $\sigma$-nice if the following two conditions hold:
\begin{enumerate}
\item  \textbf{Centering property:}  For every $\delta> 0 $, and every $\x^*_\delta \in \argmin_{\x\in\K} \fhat_{\delta}(\x)$, there exists $ \x^*_{\delta/2} \in \arg\min _{x\in\K}\fhat_{\delta/2}(\x)$, such that:
$$\|\x^*_\delta - \x^*_{\delta/2}\| \leq \frac{\delta}{2}$$ 
\item  \textbf{Local strong convexity of the smoothed function:} For every $\delta> 0 $ let $r_\delta=3\delta$, and denote $\x^*_\delta = \argmin_{\x\in\K} \fhat_{\delta}(\x)$,
  then over $B_{  r_\delta} (\x^*_\delta)$, the function $\fhat_{\delta}(\x)$
  is $\sigma$-strongly-convex.
\end{enumerate}
\end{definition}

Hence, $\sigma$-nice is a combination of two properties.
Both together  imply  that optimizing the smoothed version on a scale $\delta$ is a good start for optimizing a finer version on a scale of $\delta/2$, which is sufficient  for a scheme based on graduated optimization to work as we show next. 
In Section~\ref{section:experiments} we show that  $\sigma$-nice functions arise naturally in data.  An illustration of $\sigma$-nice function in $1$-dimension appears in Figure~\ref{fig:sigmaNice}.

 \begin{figure}[t]
 \centering
 \includegraphics[trim = 40mm 10mm 15mm 5mm, clip, width=0.35\textwidth ]{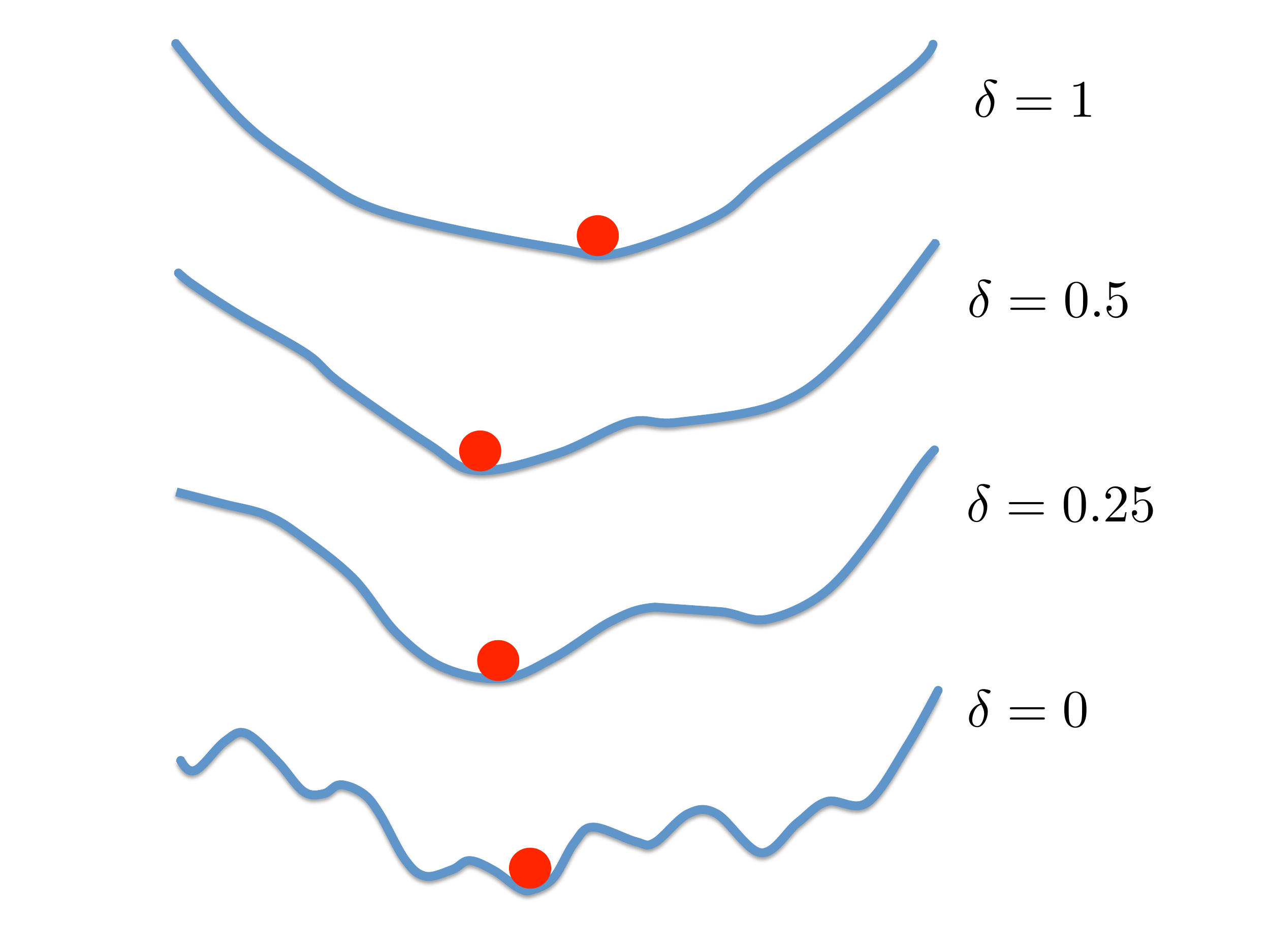}
 \caption{A $1$-dim $\sigma$-nice function ($\delta=0$), and its smoothed versions. } 
 \label{fig:sigmaNice}
 \end{figure}

\section{Graduated Optimization with a  Gradient Oracle}
\label{sec:SigmaNiceOptGrad}
In this section we assume that we can access a noisy gradient oracle for $f$. 

Thus, given $\x\in\reals^d, \delta\geq 0$ we can use $\text{SGO}_G$ (Figure~\ref{fig:SGO_G}) to obtain an unbiased and bounded estimate for $\nabla \fhat_\delta(\x)$, as ensured by Lemma~\ref{lem:Unbiasedness1}. 
Note that for ease of notation $\text{SGO}_G$ (Figure~\ref{fig:SGO_G}) is listed using an exact gradient oracle for $f$. As described at the end of Section~\ref{sec:ImplicitSmoothing}, this could be replaced with a noisy gradient oracle for $f$, and Lemma~\ref{lem:Unbiasedness1}, will still hold.

Following is our main Theorem:
\begin{theorem} \label{thm:Main}
Let $\eps \in (0,1)$ and $p\in(0,1/e)$, also let $\K$ be a convex set, and  $f$ be an $L$-Lipschitz $\sigma$-nice function.
Suppose that we apply Algorithm~\ref{alg:generic},  
then after  $\tilde{O}(1/\sigma^2\eps^2)$ rounds  Algorithm~\ref{alg:generic} outputs a point $\bar{\x}_{M+1}$ which is $\eps$ optimal with a probability greater than $1-p$.
\end{theorem}

\begin{algorithm}[t]
\caption{$\text{GradOpt}_G$  }
    \begin{algorithmic}
    \STATE \textbf{Input}:  target error $\eps$, maximal failure probability $p$, decision set $\K$
    \STATE  Choose $\bar{\x}_1 \in \K $ uniformly at random. 
     \STATE  Set $\delta_1 =  \textrm{diam}(\K)/2$, $\tilde{p} = p/M$, and $M= \log_2 \frac{1}{ \alpha_0\eps}$ where  
     $\alpha_0 = \min\{\frac{1}{2L\textrm{diam}(\K)},\frac{2\sqrt{2}}{\sqrt{\sigma} \textrm{diam}(\K)} \}$
    	\FOR { $m=1$ to $M$} 
	\STATE // Perform SGD over $\fhat_{\delta_m}$
	\STATE Set $\eps_m: = \sigma \delta_m^2/32$, and 
	$$T_F = \frac{12480L^2}{\sigma \eps_m}\log\big(\frac{2}{\tilde{p}}+2\log\frac{12480L^2}{\sigma\eps_m}  \big) $$
	\STATE Set shrinked decision set,
	$$\K_{m} : = \K\cap B(\bar{x}_m,1.5\delta_m)$$ 
	\STATE Set gradient oracle for $\fhat_{\delta_m}$,
	 $$\text{GradOracle}(\cdot) =\text{SGO}_G(\cdot,\delta_m)$$
	\STATE Update:
	 $$\bar{\x}_{m+1} \gets \text{Suffix-SGD}(T_F ,\K_m, \bar{\x}_{m},\text{GradOracle} )$$
            \STATE $\delta_{m+1} = \delta_m/2$
            \ENDFOR
            \STATE \textbf{Return}: $\bar{\x}_{M+1}$
    \end{algorithmic}
   \label{alg:generic}
   \end{algorithm}

\begin{algorithm}[h!]
\caption{Suffix-SGD  } 
    \begin{algorithmic}
    \STATE  \textbf{Input}: total time $T_F$, decision set $\K$,  initial point $\x_1 \in \K$, gradient oracle $\text{GradOracle}(\cdot)$
    	\FOR { $t=1$ to $T_F$} 
	     \STATE Set $\eta_t = 1/\sigma t$ 
	      \STATE Query the gradient oracle at $\x_t$: 
	                   $$g_t \gets \text{GradOracle}(\x_t) $$
	      \STATE Update: 
	              $\x_{t+1} \gets \Pi_{\K}(\x_t-\eta_t g_t) $
	\ENDFOR
     \STATE \textbf{Return}: $\bar{\x}_{T_F}: = \frac{2}{ T_F}\big(\x_{ T_F/2+1}+\ldots+\x_{T_F} \big)$
    \end{algorithmic}
    \label{Alg:SuffixGD}
   \end{algorithm}

Algorithm~\ref{alg:generic} is divided into epochs, at  epoch $m$ it uses  $\text{SGO}_G$ to obtain unbiased estimates for the gradients of $\fhat_{\delta_m}$ which are then employed by $\text{Suffix-SGD}$ (Algorithm~\ref{Alg:SuffixGD}), to optimize this smoothed version. This optimization over $\fhat_{\delta_m}$ is performed until we are ensured to reach a point close enough to $\x^*_{m+1}:=\argmin_{\x\in\K} \fhat_{\delta_{m+1}}(\x)$, i.e., the minimum of $\fhat_{\delta_{m+1}}$.
Also note that at epoch $m$ the optimization over $\fhat_{\delta_m}$ is initialized at $\bar{\x}_{m}$ which is the point reached at the previous epoch.

$\text{Suffix-SGD}$ (Algorithm~\ref{Alg:SuffixGD}),  is a stochastic optimization algorithm for 
strongly convex functions. Its guarantees are presented in Section~\ref{sec:AnalysisGrad}.

\subsection{Analysis}
\label{sec:AnalysisGrad}
Let us first discuss $\text{Suffix-SGD}$ (Algorithm~\ref{Alg:SuffixGD}).  This algorithm performs projected gradient descent using the gradients received from  $\text{GradOracle}(\cdot)$. The  projection operator
$\Pi_\K$, is defined $\forall \y\in\reals^d$ as $$\Pi_\K(\y):=\argmin_{\x\in\K}\|\x-\y\|~.$$

Now consider  a $\sigma$-strongly convex function $F:\K \to \reals$, and suppose that we have an oracle, $\text{GradOracle}(\cdot)$, that upon querying 
a point $\x\in \K$ returns an unbiased and bounded gradient estimate, $g$, i.e.,  $\E[g] = \nabla F(x)$, and  $\|g\|\leq G$.
Note the following result from \cite{rakhlin2011making} regarding stochastic optimization of $\sigma$-strongly-convex functions, given such an oracle:

\begin{theorem} \label{thm:Shamir}
Let $p\in (0,1/e)$, and $F$ be a $\sigma$-strongly convex function. Suppose that $\text{GradOracle}(\cdot)$ produces G-bounded, and unbiased estimates of $\nabla F$.
Then after no more than $T_F$ rounds, the final point $ \bar{\x}_{T_F}$ returned by Suffix-SGD (Algorithm~\ref{Alg:SuffixGD} ) ensures that with a probability $\geq 1-p$, we have:
 $$F(\bar{\x}_{T_F}) - \min_{\x\in \K} F(\x) \leq \frac{6240\log\big(2\log(T_F)/p\big)G^2}{\sigma T_F}  $$
\end{theorem}
\begin{corollary}
\label{cor:Shamir}
The latter means that for $T_F \geq \frac{12480G^2}{\sigma \eps}\log\big(2/p+2\log(12480G^2/\sigma\eps)  \big)$ we will have an excess loss smaller than $\eps$.
\end{corollary}

Notice that at each epoch $m$ of $\text{GradOpt}_G$, it initiates $\text{Suffix-SGD}$ with a gradient oracle  $\text{SGO}_G(\cdot,\delta_m)$. According to Lemma~\ref{lem:Unbiasedness1}, $\text{SGO}_G(\cdot,\delta_m)$ produces an unbiased and $L$-bounded estimates of $\fhat_{\delta_m}$, thus in the analysis of each epoch we can use 
Theorem~\ref{thm:Shamir} for $\fhat_{\delta_m}$, taking $G=L$.

Following is our key Lemma:
\begin{lemma} \label{lem:induction}
Consider $M$, $\K_m$ and  $\bar{\x}_{m+1}$ as defined in Algorithm~\ref{alg:generic}. Also denote by $\x^*_m$ the minimizer of $\fhat_{\delta_m}$ in $\K$. Then the following  holds for all  $1\leq m \leq M$ w.p.$\geq 1-p$:
\begin{enumerate}
\item The smoothed version $\fhat_{\delta_m} $is $\sigma$-strongly convex over $\K_m$, and $\x_{m}^* \in \K_m$.
\item Also,  $\fhat_{\delta_m}(\bar{\x}_{m+1}) - \fhat_{\delta_m}(\x_m^*) \leq \sigma \delta_{m+1}^2/8$
\end{enumerate}
\end{lemma}

\begin{proof}
We will prove by induction, let us prove it holds for $m=1$. Note that $\delta_1 =  \textrm{diam}(\K)/2$, therefore $\K_1=\K$, and also $\x_1^*\in\K_1$.
 Also recall that $\sigma$-niceness of $f$ implies that $\fhat_{\delta_1}$ is $\sigma$-strongly convex in $\K$, thus
by Corollary~\ref{cor:Shamir}, after less than $T_F = \tilde{\mathcal{O}}(\frac{12480L^2}{\sigma (\sigma \delta_1^2/32)}) $ optimization steps of Suffix-SGD with a probability greater than $1-p/M$, we will have:
$$\fhat_{\delta_1}(\bar{\x}_2) - \fhat_{\delta_1}(\x_1^*) \leq \sigma \delta_1^2/32= \sigma \delta_{2}^2/8 $$
which establishes the case of $m=1$.
Now assume that lemma holds for $m>1$. By this assumption, $\fhat_{\delta_m}(\bar{\x}_{m+1}) - \fhat_{\delta_m}(\x_m^*) \leq \sigma \delta_{m+1}^2/8$,  $\fhat_{\delta_m}$ is $\sigma$-strongly convex in $\K_m$, and also  $\x_m^* \in \K_m$. Hence, we can use  Equation~\eqref{eq:strongCvxity}  to get:
$$\|\bar{\x}_{m+1} - \x_m^* \|\leq  \sqrt{\frac{2}{\sigma}} \sqrt{\fhat_{\delta_m}(\bar{\x}_{m+1}) - \fhat_{\delta_m}(\x_m^*)} = \frac{\delta_{m+1}}{2}$$
Combining the latter with the centering property of $\sigma$-niceness yields:
\begin{align*}
\|\bar{\x}_{m+1} - \x_{m+1}^* \|&\leq \|\bar{\x}_{m+1} - \x_m^* \|+\| \x_m^* -\x_{m+1}^*\| \\
& \leq 1.5\delta_{m+1} 
\end{align*}
and it follows that,
$$\x_{m+1}^* \in B(\bar{\x}_{m+1},1.5\delta_{m+1})\subset B({\x}_{m+1}^*,3\delta_{m+1})$$
 Recalling that $\K_{m+1}: = B(\bar{\x}_{m+1},1.5\delta_{m+1})$, and the local strong convexity property of $f$ (which is $\sigma$-nice), then the induction step for first part of the lemma holds. Now, by Corollary~\ref{cor:Shamir}, after less than $T_F = \tilde{\mathcal{O}}(\frac{12480L^2}{\sigma (\sigma \delta_{m+1}^2/32)}) $ optimization steps of Suffix-SGD over $\fhat_{\delta_{m+1}}$, we will have:
$$\fhat_{\delta_{m+1}}(\bar{\x}_{m+2}) - \fhat_{\delta_{m+1}}(\x_{m+1}^*) \leq \sigma \delta_{m+2}^2/8$$
which establishes the induction step for the second part of the lemma.

An analysis of fail probability: since we have $M$ epochs in total and at each epoch the fail probability is smaller than $p/M$, then the total fail probability of our algorithm is smaller than $p$.
\end{proof}
We are now ready to prove Theorem~\ref{thm:Main}:
\begin{proof}[proof of Theorem~\ref{thm:Main}]
Algorithm~\ref{alg:generic} terminates after $M = \log_2 \frac{1}{\alpha_0 \eps}$ epochs meaning, $\delta_{M} =  \textrm{diam}(\K)\alpha_0 \eps/2$.  According to Lemma~\ref{lem:induction} the following holds w.p.$\geq 1-p$~, for every $\x\in \K$,
\begin{align*}
\fhat_{\delta_{M}}(\bar{\x}_{M+1}) - \fhat_{\delta_{M}}(\x)& \leq \sigma \delta_{M+1}^2/8 \\
&=\left( \frac{\sqrt{\sigma} \textrm{diam}(\K)\alpha_0\eps}{8\sqrt{2} }\right)^2 
\end{align*} 
Due to Lemma~\ref{lem:SmoothingLemma}, $\fhat_{\delta_{M}}$ is $L\delta_{M}$ biased from $f$,  thus for every $\x\in \K$,
\begin{align*}
f(\bar{\x}_{M+1}) - f(\x) 
&\leq L\textrm{diam}(\K)\alpha_0 \eps+\left( \frac{\sqrt{\sigma} \textrm{diam}(\K)\alpha_0\eps}{8\sqrt{2} }\right)^2  \\
&\leq \eps
\end{align*}
we used $\alpha_0 = \min\{\frac{1}{2L\textrm{diam}(\K)},\frac{2\sqrt{2}}{\sqrt{\sigma} \textrm{diam}(\K)} \}$, and  $\eps< 1$.

Let $T_{\text{total}}$, be the total number of queries made  by Algorithm~\ref{alg:generic}, then we have:
\begin{align*}
T_{\text{total}} &\leq  \sum_{m=1}^{M} \frac{12480L^2}{\sigma \eps_m}\log\Gamma \\
&\leq \sum_{m=1}^{M} \frac{12480L^2}{\sigma  (\sigma \delta_m^2/32)}\log\Gamma \\
&\leq \frac{4\cdot 10^5 L^2\log\Gamma }{\sigma^2}  \sum_{i=1}^{M} \frac{4^{i-1}}{\delta_1^2 }\\
&\leq  \frac{14 \cdot 10^4 L^2\log\Gamma }{\sigma^2}  \frac{4^{M}}{\delta_1^2 } \\
& \leq   \frac{14 \cdot 10^4 L^2\log\Gamma }{\sigma^2}  \max\{ {16L^2},{\sigma}/{2} \} \frac{1}{\eps^2} 
\end{align*}
here we used the notation:
\begin{align*}
\Gamma&: = \frac{2M}{p}+2\log(12480L^2/\sigma\eps_M) \\
&\leq  \frac{2M}{p}+2\log(4 \cdot 10^5 L^2\max\{ {16L^2},\frac{\sigma}{2} \}/\sigma^2 \eps^2) 
\end{align*}
\end{proof}

\section{Graduated Optimization  with  a Value Oracle}
\label{sec:SigmaNiceOptVal}
In this section we assume that we can access a noisy value oracle for $f$. 

Thus, given $\x\in\reals^d, \delta\geq0$ we can use $\text{SGO}_V$ (Figure~\ref{fig:SGO_V}) as an oracle that outputs an unbiased and bounded estimates for $\nabla \fhat_\delta(\x)$, as ensured by Lemma~\ref{lem:Unbiasedness2}. 
Note that for ease of notation $\text{SGO}_V$ (Figure~\ref{fig:SGO_V}) is listed using an exact value oracle for $f$. As described at the end of Section~\ref{sec:ImplicitSmoothing}, this could be replaced with a noisy value oracle for $f$, and Lemma~\ref{lem:Unbiasedness2}, will still hold.

Following is our main Theorem:
\begin{theorem} \label{thm:Main2}
Let $\eps>0$ and $p\in(0,1/e)$, also let $\K$ be a convex set, and  $f$ be an $L$-Lipschitz $\sigma$-nice function. Assume also that $\max_\x |f(\x)|\leq C$. 
Suppose that we apply Algorithm~\ref{alg:generic2},  
then after after $\tilde{O}(d^2/\sigma^2\eps^4)$ rounds  Algorithm~\ref{alg:generic2} outputs a point $\bar{\x}_{M+1}$ which is $\eps$ optimal with a probability greater than $1-p$.
\end{theorem}
\begin{algorithm}[t] 
\caption{$\text{GradOpt}_V$  }
    \begin{algorithmic}
    \STATE \textbf{Input}:  target error $\eps$, maximal failure probability $p$, decision set $\K$
    \STATE  Choose $\bar{\x}_1 \in \K $ uniformly at random. 
     \STATE  Set $\delta_1 =  \textrm{diam}(\K)/2$, $\tilde{p} = p/M$, and $M= \log_2 \frac{1}{ \alpha_0\eps}$ where  
     $\alpha_0 = \min\{\frac{1}{2L\textrm{diam}(\K)},\frac{2\sqrt{2}}{\sqrt{\sigma} \textrm{diam}(\K)} \}$
    \FOR { $m=1$ to $M$} 
	\STATE // Perform SGD over $\fhat_{\delta_m}$
	\STATE Set $\eps_m: = \sigma \delta_m^2/32$, and 
	$$T_F = \frac{12480}{\sigma \eps_m}\frac{d^2C^2}{\delta_m^2}\log\big(\frac{2}{\tilde{p}}+2\log\frac{12480d^2C^2}{\sigma\eps_m\delta_m^2}  \big) $$
	\STATE Set shrinked decision set,
	$$\K_{m} : = \K\cap B(\bar{x}_m,1.5\delta_m)$$ 
	\STATE Set gradient oracle for $\fhat_{\delta_m}$,
	 $$\text{GradOracle}(\cdot) =\text{SGO}_V(\cdot,\delta_m)$$
	\STATE Update:
	 $$\bar{\x}_{m+1} \gets \text{Suffix-SGD}(T_F ,\K_m,\bar{\x}_{m}, \text{GradOracle} )$$
            \STATE $\delta_{m+1} = \delta_m/2$
            \ENDFOR
            \STATE \textbf{Return}: $\bar{\x}_{M+1}$
    \end{algorithmic}
   \label{alg:generic2}
   \end{algorithm}

\subsection{Analysis}
\label{sec:AnalysisValue}
Notice that at each epoch $m$ of $\text{GradOpt}_V$, it initiates $\text{Suffix-SGD}$ with a gradient oracle  $\text{SGO}_V(\cdot,\delta_m)$. According to Lemma~\ref{lem:Unbiasedness2}, $\text{SGO}_V(\cdot,\delta_m)$ produces an unbiased and ${d C}/{\delta_m}$-bounded estimates for the gradients of $\fhat_{\delta_m}$, thus in the analysis of each epoch we can use 
Corollary~\ref{cor:Shamir} for $\fhat_{\delta_m}$, taking $G={d C}/{\delta_m}$.

Following is our key Lemma:
\begin{lemma} \label{lem:induction2}
Consider $M$, $\K_m$ and  $\bar{\x}_{m+1}$ as defined in Algorithm~\ref{alg:generic2}. Also denote by $\x^*_m$ the minimizer of $\fhat_{\delta_m}$ in $\K$. Then the following  holds for all  $1\leq m \leq M$ w.p.$\geq 1-p$:
\begin{enumerate}
\item The smoothed version $\fhat_{\delta_m} $is $\sigma$-strongly convex over $\K_m$, and $\x_{m}^* \in \K_m$.
\item Also,  $\fhat_{\delta_m}(\bar{\x}_{m+1}) - \fhat_{\delta_m}(\x_m^*) \leq \sigma \delta_{m+1}^2/8$
\end{enumerate}
\end{lemma}

The proof of Lemma~\ref{lem:induction2} is similar to the proof of Lemma~\ref{lem:induction} given in Section~\ref{sec:AnalysisGrad},  we therefore omit the details.

We are now ready to prove Theorem~\ref{thm:Main2}:
\begin{proof}[proof of Theorem~\ref{thm:Main2}]
Let $\bar{\x}_{M+1}$ be the output of 
Algorithm~\ref{alg:generic2}.  Similarly to the proof of Theorem~\ref{thm:Main}, we can show that for every $\x\in\K$:
\begin{align*}
f(\bar{\x}_{M+1}) - f(\x) \leq \eps
\end{align*}
Let $T_{\text{total}}$, be the total number of queries made by  Algorithm~\ref{alg:generic2}, then we have:
\begin{align*}
T_{\text{total}} &\leq  \sum_{m=1}^{M} \frac{12480d^2 C^2}{\sigma \eps_m \delta_m^2}\log\Gamma \\
&\leq \sum_{m=1}^{M} \frac{12480d^2C^2}{\sigma  (\sigma \delta_m^2/32)\delta_m^2}\log\Gamma \\
&\leq \frac{4\cdot 10^5 d^2C^2\log\Gamma }{\sigma^2}  \sum_{i=1}^{M} \frac{8^{i-1}}{\delta_1^4 }\\
&\leq  \frac{6 \cdot 10^4 d^2C^2\log\Gamma }{\sigma^2}  \frac{8^{M}}{\delta_1^4 } \\
& \leq   \frac{6 \cdot 10^4 d^2C^2\log\Gamma }{\sigma^2}  \max\{ {256 L^4},{\sigma^2}/{4} \} \frac{1}{\eps^4} 
\end{align*}
here we used the notation:
\begin{align*}
\Gamma&: = \frac{2M}{p}+2\log(12480d^2C^2/\sigma\eps_M \delta_M^2) \\
&\leq \frac{2M}{p}+2\log(4 \cdot 10^5 d^2C^2  \max\{ {256 L^4},\frac{{\sigma^2}}{4} \}/\sigma^2 \eps^4) 
\end{align*}
\end{proof}

\section{Experiments}
\label{section:experiments}
In the last two decades, performing complex learning tasks using Neural-Network (NN) architectures has become an active and promising line of research.  
Since learning NN architectures essentially requires to solve a hard non-convex program, we have decided to focus our empirical study on this  type of tasks.
As a test case, we train a NN with a single hidden layer of
  $30$ units over the MNIST data set. We adopt the experimental setup of \cite{BengioDauphin2014identifying}  and train over a down-scaled version of the data, i.e., the original $28\times 28$ images of MNIST were down-sampled to the size of $10\times 10$. We use  a ReLU activation function, and  minimize the square loss.

\begin{figure}[h]
\centering
\subfigure[]{ 
\includegraphics[trim = 15mm 67mm 17mm 69mm, clip,width=0.3\textwidth ]{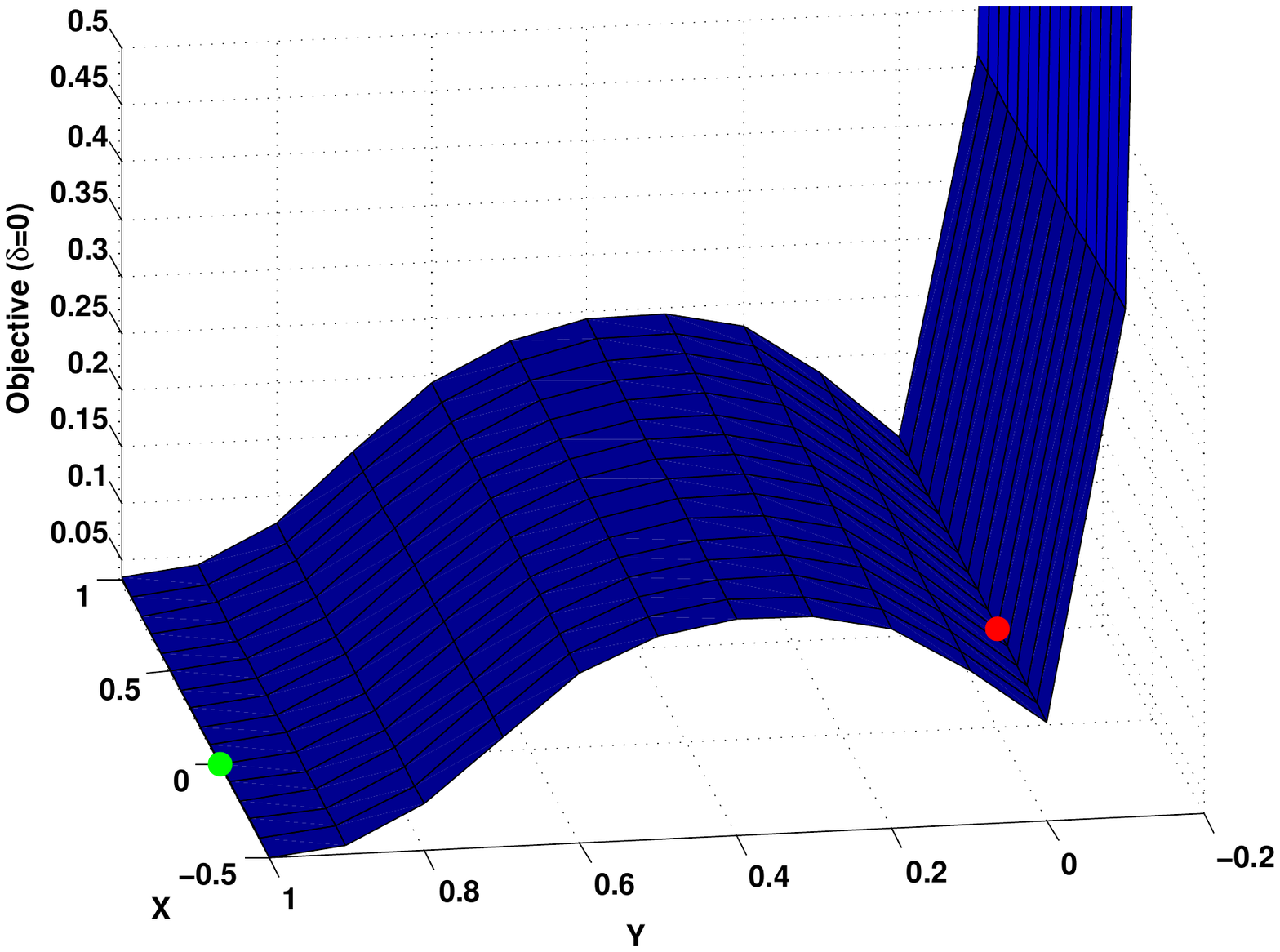}}
\subfigure[]{
 \includegraphics[trim = 15mm 67mm 17mm 69mm, clip,width=0.3\textwidth ]{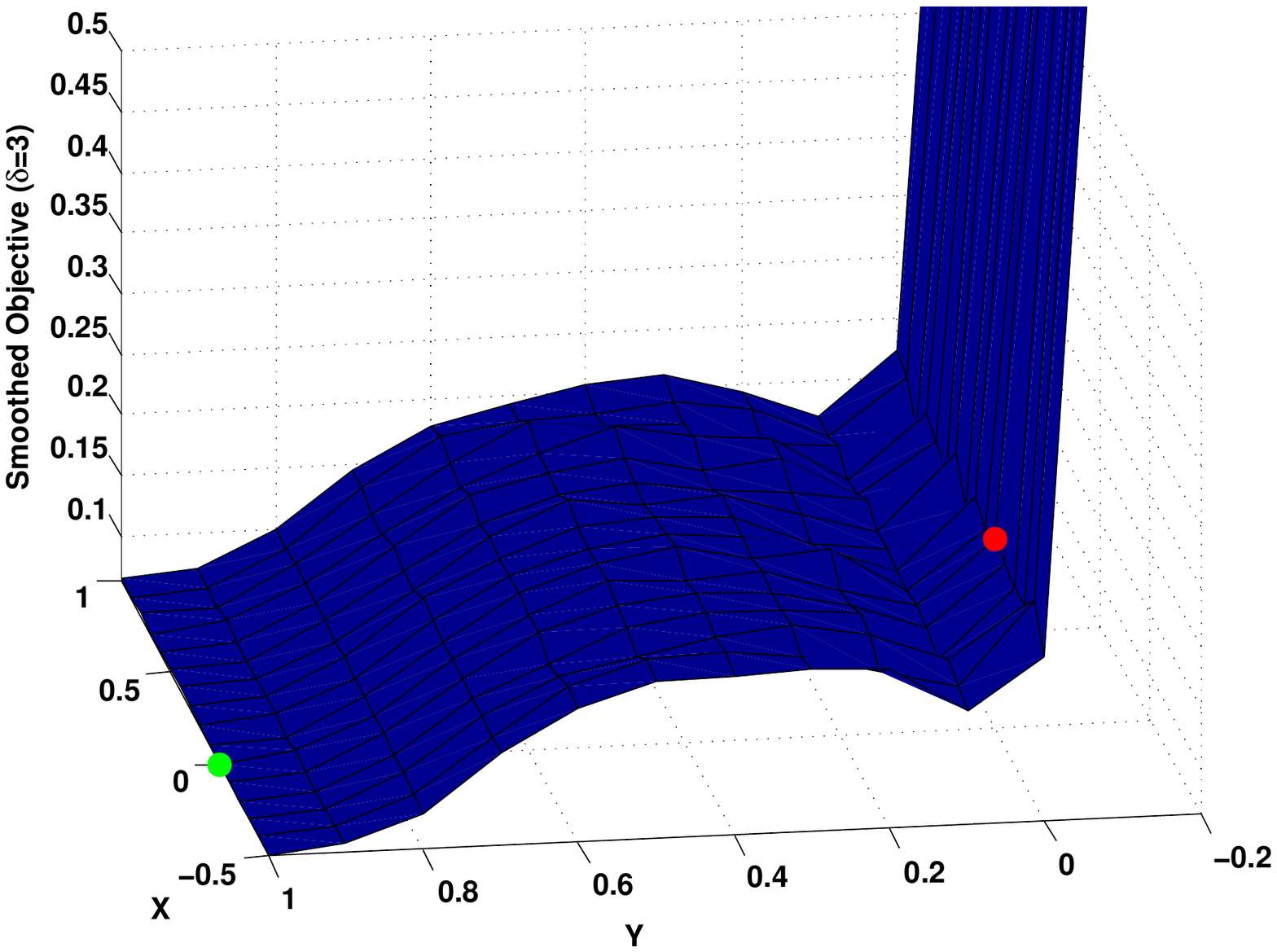}}
 \subfigure[]{
\includegraphics[trim = 15mm 67mm 17mm 69mm, clip, width=0.3\textwidth ]{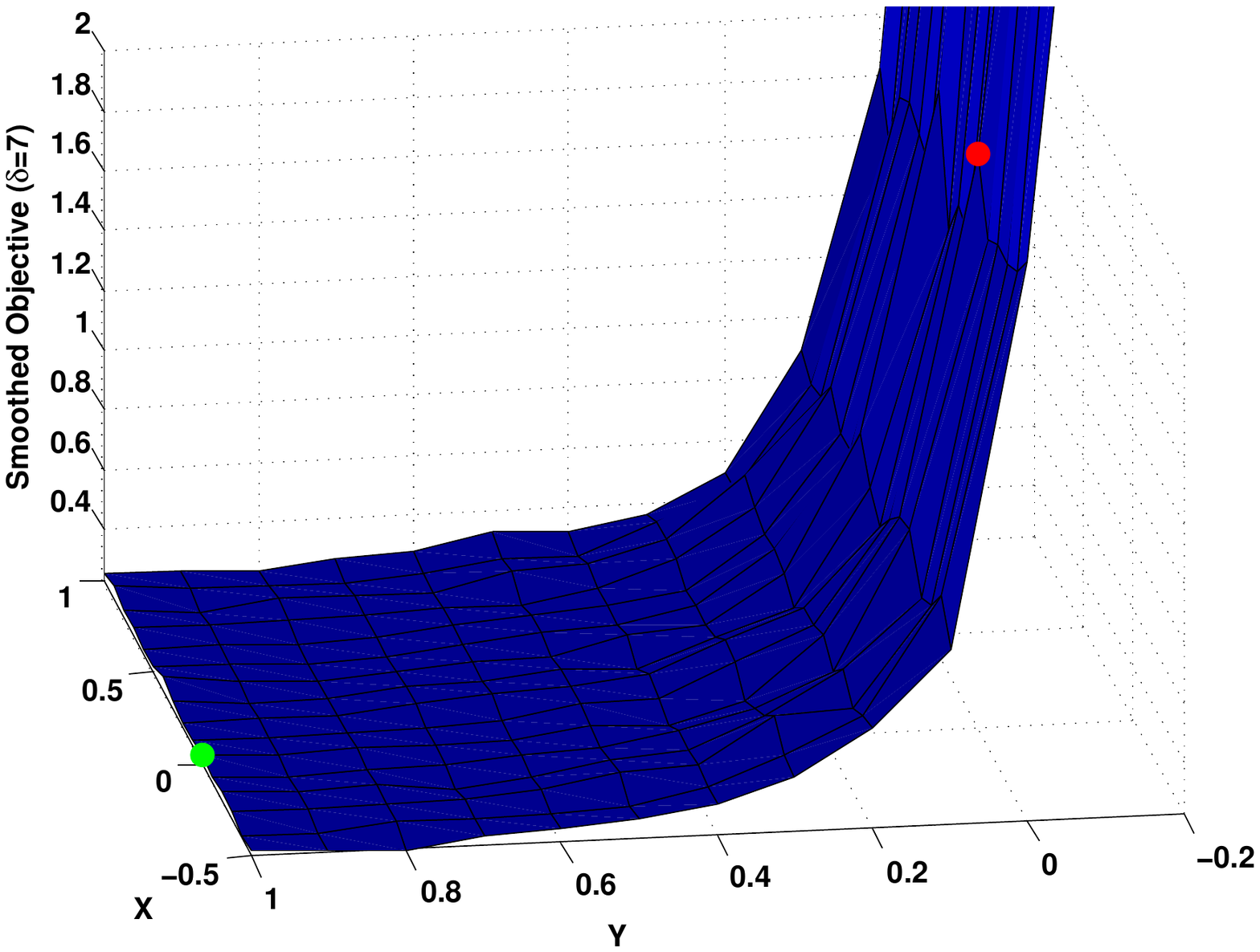}}
\caption{The objective near a stall point. Left: $\delta=0$. Middle:
  $\delta=3$. Right: $\delta=7$.} 
\label{fig:Object3D}
\end{figure}

%

\subsection{Smoothing the NN }
First, we were interested in exploring the non-convex structure of the
above NN learning task, and check whether our definition of
$\sigma$-nice complies with this structure.  We started by running
MSGD (Minibatch Stochastic Gradient Descent) on the problem, while
using a batch size of $100$, and a step size rule of the form $\eta_t
= \eta_0(1+\gamma t)^{-3/4}$, where $\eta_0 = 0.01,\; \gamma =
10^{-4}$. This choice of step size rule was the most effective among a
grid of rules that we examined.  We have found out that  MSGD frequently
``stalls" in areas with a relatively high loss, here we relate to
points at the end of such run as stall-points.

In order to learn about the non-convex nature of the problem, we examined the objective values along two directions around  stall-points.
The first direction was the gradient at the stall point, and the
second direction was the line connecting the stall-point to $\x^*$,
where $\x^*$ is the best weights configuration of the NN that we were
able to find. A drawing depicting typical results appears on the left
side of Figure~\ref{fig:Object3D}.
The stall-point appears in red, and $\x^*$  in green; also the axis marked as $X$ is the gradient direction, and one marked $Y$ is the direction between stall-point and $\x^*$. Note that the stall-point is inside a narrow ``valley", which prevents MSGD from ``seeing"  $\x^*$, and so 
it seems that MSGD slowly progresses downstream. Interestingly, the objective  around $x^*$ seems strongly-convex in the direction of  the stall point.

On the middle of Figure~\ref{fig:Object3D}, we present the $\delta=3$
smoothed version of the same objective that appears on the left side of Figure~\ref{fig:Object3D}.
We can see that the ``valley" has not vanished, but the gradient of the smoothed version leads us slightly towards $\x^*$ and out of the original ``valley".
On the right side of Figure~\ref{fig:Object3D}, we present the
$\delta=7$ smoothed version of  the objective. 
We can see that due to the coarse smoothing, the ``valley" in which MSGD was stalled, has completely dissolved, and the gradient of this version leads us towards $\x^*$.

\subsection{Graduated Optimization of NN }
Here we present experiments that demonstrate the effectiveness of $\text{GradOpt}_G$ (Algorithm~\ref{alg:generic}) in training  the NN mentioned above. 
First, we wanted to learn if smoothing can help us escape points where MSGD stalls. We used MSGD ($\delta=0$) to train the NN, and as before 
we found that its progress slows down,  yielding  relatively high error. We then took the point that MSGD reached after $5\cdot10^4$ iteration and initialized an optimization over the smoothed versions of the loss; this was done using smoothing values of $\{1,3,5,7\}$.
In Figure~\ref{fig:InitDelta} we present the results of the above experiment. 

\begin{figure}[t]
\centering
\includegraphics[trim = 2mm 55mm 11mm 60mm, clip, width=0.45\textwidth ]{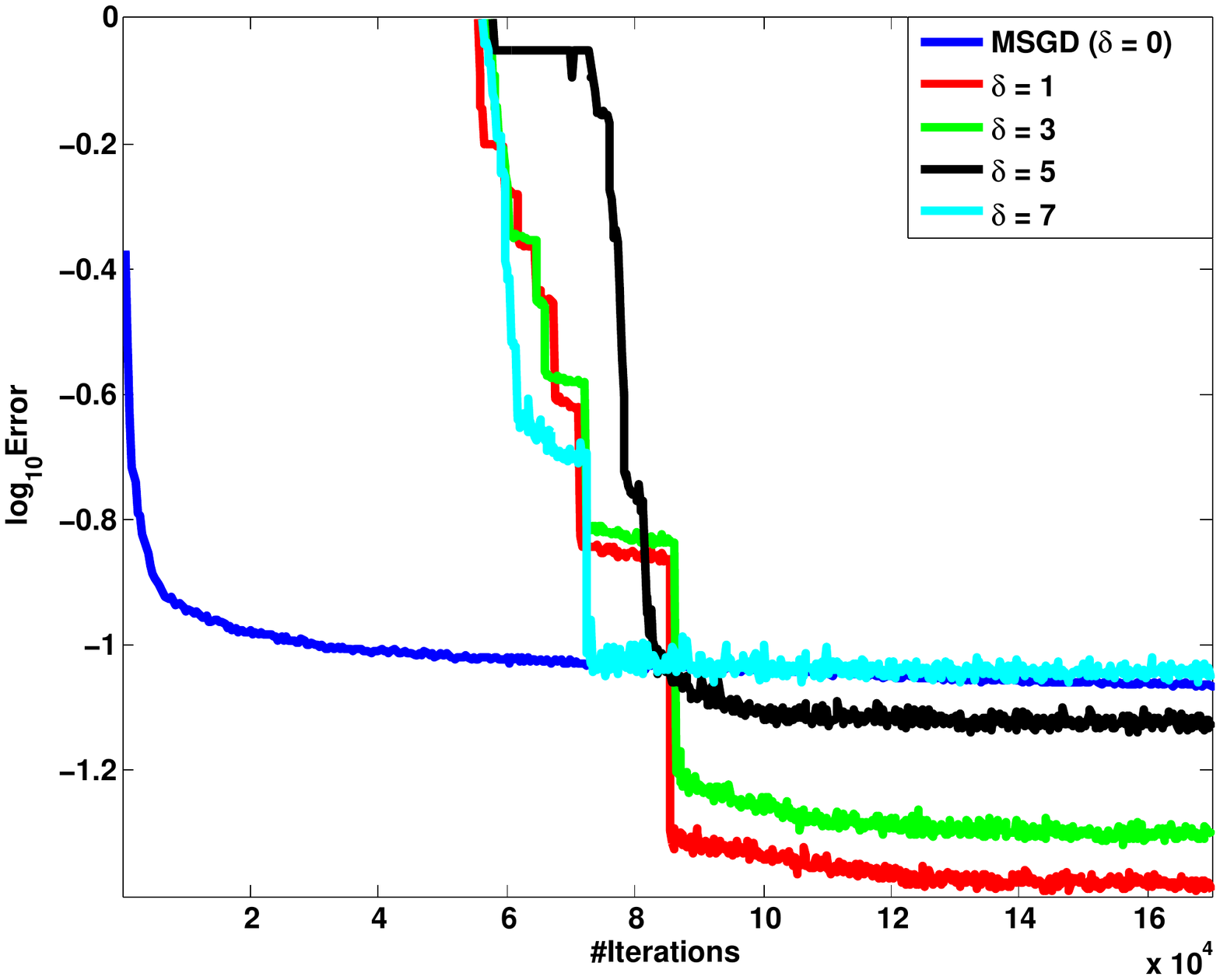}
\caption{Running optimization with  fixed smoothing values, starting at the point where MSGD stuck after $5\cdot10^4$ iterations. } 
\label{fig:InitDelta}
\end{figure}
As seen in Figure~\ref{fig:InitDelta}, small $\delta$'s converge slower than large $\delta$'s, but produce a much better solution.   
Furthermore, the initial optimization progresses in leaps, for large $\delta$'s the leaps are sharper, and lower $\delta$'s demonstrate smaller leaps. We believe that these leaps are associated with the advance of the optimization from one local ``valley" to another;  
Larger values of  $\delta$  dissolve the ``valleys" much easily, but converge to points with higher errors than small $\delta$'s, due to the increase of the bias with smoothing.

In Figure~\ref{fig:GradOpt}  we compare our complete graduated optimization algorithm, namely $\text{GradOpt}_G$ (Alg.~\ref{alg:generic}) 
 to MSGD. We started with an initial smoothing of $\delta=7$, which decayed according to $\text{GradOpt}_G$.
 Note that $\text{GradOpt}_G$ progresses very fast and yields a  much better solution than MSGD.

\begin{figure}[h]
\centering
\includegraphics[trim = 2mm 57mm 11mm 65mm, clip, width=0.45\textwidth ]{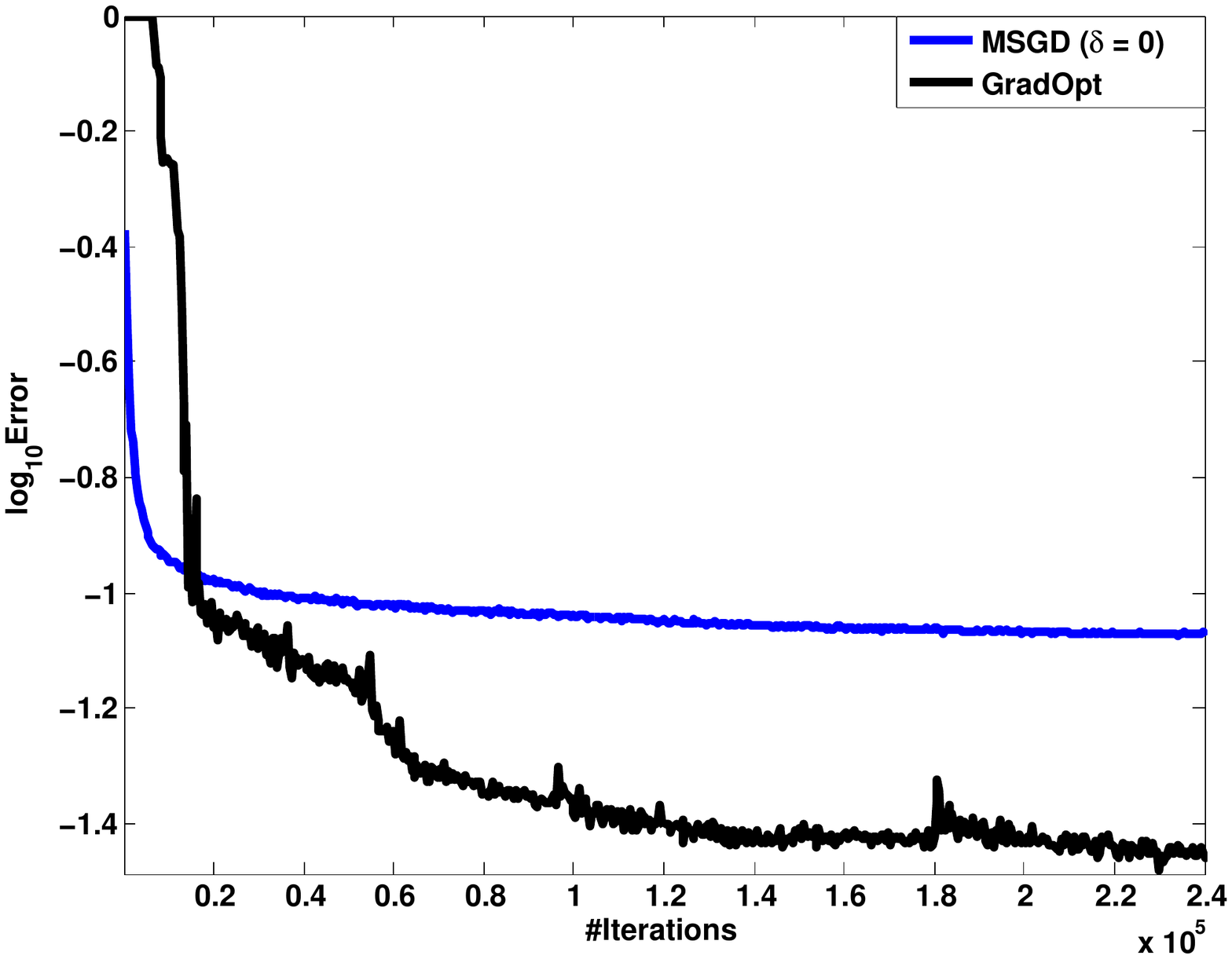}
\caption{Comparison between MSGD and $\text{GradOpt}_G$.} 
\label{fig:GradOpt}
\end{figure}

\section{Discussion}

We have described a family of non-convex functions which admit efficient optimization via the graduated optimization methodology, and gave the first rigorous analysis of a first-order algorithm in the stochastic setting.   

We view it as only a first glimpse of the potential of graduated optimization to provable non-convex optimization, and amongst the interesting questions that remain we find
\begin{itemize}
\item
Is $\sigma$-niceness necessary for convergence of first-order methods to a global optimum? Is there a more lenient property that better captures the power of graduated optimization?  
\item
Amongst the two properties of $\sigma$-niceness, can their parameters be relaxed in terms of the ratio of smoothing to strong-convexity, or to centering? 
\item
Can second-order/other methods give rise to better convergence rates / faster algorithms for stochastic or offline $\sigma$-nice non-convex optimization? 
\end{itemize}

\bibliographystyle{abbrvnat}
\bibliography{bib}

\newpage
\appendix

\end{document}